\documentclass[letterpaper, 10 pt, conference]{ieeeconf}  
\IEEEoverridecommandlockouts
\pdfoutput=1

\usepackage{graphicx} 
\usepackage{float}
\usepackage{setspace}
\usepackage{amsmath}
\usepackage{mathrsfs}
\usepackage{amssymb}
\usepackage{subcaption}
\usepackage[noadjust]{cite}
\usepackage{multirow}
\usepackage{array}
\usepackage{inputenc}
\let\labelindent\relax
\usepackage{enumitem}
\usepackage{hyperref}
\usepackage[usenames, dvipsnames]{color}
\usepackage{soul}

\usepackage{tikz}

\usepackage{amsthm}
\usetikzlibrary{positioning, calc, arrows}
\usetikzlibrary{shapes.geometric,shapes.arrows,decorations.pathmorphing}
\usetikzlibrary{matrix,chains,scopes,fit,arrows.meta}

\theoremstyle{plain}
\newtheorem{theorem}{{Theorem}}
\newtheorem{lemma}{{Lemma}}
\newtheorem{assumption}{{Assumption}}

\newcommand{\norm}[1]{\left\lVert#1\right\rVert}
    
\DeclareCaptionLabelSeparator{periodspace}{.\quad}
\newcommand{\bbm}{\begin{bmatrix}}
\newcommand{\ebm}{\end{bmatrix}}

\newcommand{\Real}{\mathbb{R}}

\newcommand{\lone}{{\mathcal{L}_1}}
\newcommand{\linf}{{\mathcal{L}_\infty}}

\title{\LARGE \bf 
High-Precision Trajectory Tracking in Changing Environments Through $\lone$ Adaptive Feedback and Iterative Learning }

\author{Karime Pereida, Rikky R. P. R. Duivenvoorden, and Angela P. Schoellig
\thanks{The authors are with the Dynamic Systems Lab (www.dynsyslab.org) at the 
University of Toronto Institute for Aerospace Studies (UTIAS), Canada. Email:
        { \{karime.pereida, rikky.duivenvoorden\}@robotics.utias.utoronto.ca,} 
        { schoellig@utias.uto\-ronto.ca}. }%
\thanks{This research was supported in part by NSERC grant RGPIN-2014-04634, the Connaught New Researcher
Award, and the Mexican  National Council  of  Science  and  Technology  (abbreviated  CONACYT).}%
}

\newcommand\copyrighttext{\footnotesize \textbf{Accepted version.} Accepted at \textit{2017 IEEE International Conference on Robotics and Automation.}

\textcopyright 2017 IEEE. Personal use of this material is permitted. Permission from IEEE must be obtained for all other uses, in any current or future media, including reprinting/republishing this material for advertising or promotional purposes, creating new collective works, for resale or redistribution to servers or lists, or reuse of any copyrighted component of this work in other works.}

\newcommand\copyrightnotice{\begin{tikzpicture}[remember picture,overlay]
\node[anchor=south,yshift=10pt] at (current page.south) {\fbox{\parbox{\dimexpr\textwidth-\fboxsep-\fboxrule\relax}{\copyrighttext}}};
\end{tikzpicture}}

\begin{document}

\maketitle
\thispagestyle{empty}
\pagestyle{empty}

\copyrightnotice{} 
\begin{abstract}
As robots and other automated systems are introduced to unknown and dynamic environments, robust and adaptive control strategies are required to cope with disturbances, unmodeled dynamics and parametric uncertainties. In this paper, we propose and provide theoretical proofs of a combined $\lone$ adaptive feedback and iterative learning control (ILC) framework to improve trajectory tracking of a system subject to unknown and changing disturbances. The $\lone$ adaptive controller forces the system to behave in a repeatable, predefined way, even in the presence of unknown and changing disturbances; however, this does not imply that perfect trajectory tracking is achieved. ILC improves the tracking performance based on experience from previous executions. The performance of ILC is limited by the robustness and repeatability of the underlying system, which, in this approach, is handled by the $\lone$ adaptive controller. In particular, we are able to generalize learned trajectories across different system configurations because the $\lone$ adaptive controller handles the underlying changes in the system. We demonstrate the improved trajectory tracking performance and generalization capabilities of the combined method compared to pure ILC in experiments with a quadrotor subject to unknown, dynamic disturbances. This is the first work to show $\lone$ adaptive control combined with ILC in experiment. 

\end{abstract}

\section{INTRODUCTION}
Robots and automated systems are being increasingly deployed in unknown and dynamic environments. Operating in these environments requires sophisticated control methods that can guarantee high overall performance even in the presence of model uncertainties, unknown disturbances and changing dynamics. 
Examples of robotic applications  in these increasingly challenging environments include autonomous driving, assistive robotics and unmanned aerial vehicle (UAV) applications such as airborne package delivery. In the latter example, UAVs are required to deliver packages with different mass properties (mass, center of gravity and inertia), which influence the dynamic behavior of the UAV. Designing a controller to achieve high performance for each package is not feasible and small changes in the conditions may result in a dramatic decrease in controller performance and potential instability (see \cite{Skelton1989}, \cite{Morari1999} and \cite{Skogestad2007}). 

\begin{figure}[htb]
\centering{
\includegraphics[width=0.5\textwidth]{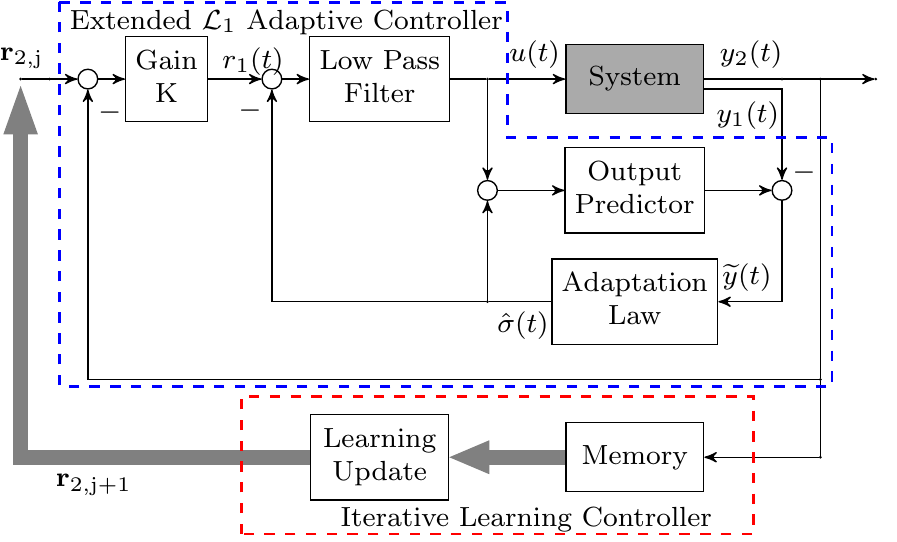}
}
\vspace{-0.5cm}
\caption{Proposed framework to achieve high performance control in changing environments. The extended $\lone$ adaptive controller forces the system to behave in a predefined, repeatable way. The iterative learning controller improves the tracking performance in each iteration $j$ based on experience from previous executions.} 
\label{fig:cascadedL1architecture}
\end{figure}

The goal of this work is to design a controller such that the system shows a repeatable and reliable behavior (that is, achieves, for the same reference input, the same output) even in the presence of unknown disturbances and changing dynamics, and improves its performance over time. In this paper, we focus on improving the trajectory tracking performance over task iterations, and propose and provide theoretical proofs of a combined $\lone$ adaptive feedback and iterative learning control (ILC) framework (see  Fig.~\ref{fig:cascadedL1architecture}). The $\lone$ adaptive controller forces the system to behave in a repeatable, predefined way, even if it is subject to model uncertainties and unknown disturbances. As a result, we obtain a repeatable system; however, perfect trajectory tracking is not achieved. To learn from previous iterations and gradually improve the trajectory tracking performance of the overall system, we implement ILC. Experimental results on a quadrotor show that the proposed approach achieves high tracking performance despite dynamic disturbances. Moreover, we show that learned trajectories can be generalized across different system configurations because the $\lone$ controller handles any (dynamic) disturbances that affect the system. 

$\lone$ adaptive control and ILC have previously been combined to improve trajectory tracking performance (see \cite{Barton2011}, \cite{Altin2013}, and \cite{Altin2014}). 
In previous work, the control input to the system ($u(t)$ in Fig.~\ref{fig:cascadedL1architecture}) was constructed by combining both $\lone$ and ILC inputs in a parallel architecture. In contrast, the serial architecture proposed in this paper places the $\lone$ adaptive control as an underlying controller, while the ILC acts as a high-level adaptation scheme that mainly compensates for systematic tracking errors. This serial architecture allows us to decouple the task of making the system behave in a predefined way even in the presence of disturbances, from the task of improving the tracking performance. Furthermore, the results presented in \cite{Barton2011}, \cite{Altin2013}, and \cite{Altin2014} are restricted to simulations while the proposed approach is the first work to show the $\lone$-ILC architecture in experiment. 


$\lone$ adaptive control is based on the model reference adaptive control (MRAC) architecture with the addition of a low-pass filter that decouples robustness from adaptation~\cite{Hovakimyan2010}. This allows arbitrarily high adaptation gains to be chosen for fast adaptation. This algorithm has been successfully implemented on UAVs to augment a baseline controller for improved disturbance rejection. Attitude control based on $\lone$ adaptive control was shown in~\cite{paper:Mallikarjunan2012}, where three algorithms were successfully implemented and tested on a quadrotor, hexacopter and octocopter, respectively. In~\cite{How2009}, $\lone$ adaptive control is implemented for a quadrotor in translational velocity output feedback control, and shows the ability of the controller to compensate for artificial reduction in the speed of a single motor. In this work, we also use $\lone$ adaptive output feedback on translational velocity, as it guarantees robustness bounds, and has a-priori known steady-state and transient performance.  

Iterative learning control efficiently uses information from previous trials to improve tracking performance within a small number of iterations by updating the feedforward input signal. ILC has successfully been applied to a variety of trajectory tracking scenarios such as motion control of industrial robot arms \cite{Gunnarsson2001} and ground vehicles \cite{Ostafew2013}, manufacturing of integrated circuits \cite{Yu2014}, swinging up a pendulum  \cite{Schoellig2009}, and quadrotor control \cite{Mueller2012}. For a survey on ILC, the reader is referred to \cite{Bristow2006}. In this paper, we use optimization-based ILC in conjunction with a model error estimator \cite{Schoellig2012}.

The remainder of this paper is organized as follows: We define the problem in Section~\ref{sec:problem}. Section~\ref{sec:methodology} details the proposed approach and proves key features such as the transient behavior of the adaptive control. Section~\ref{sec:results} shows our experimental results, including examples with changing system dynamics. We compare our approach to one with a standard underlying feedback controller. Conclusions are provided in Section~\ref{sec:conclusions}.

\section{PROBLEM STATEMENT}
\label{sec:problem}

The goal of this work is to achieve high-precision tracking despite changing system dynamics and uncertain environment conditions. The system optimizes its performance, for a given desired trajectory, over multiple executions of the task. We aim to design an algorithm that does not require to re-learning if the system dynamics continue to change. 

For simplicity of presentation, we assume the uncertain and changing system dynamics (`System' block in Fig.~\ref{fig:cascadedL1architecture}) can be described by a single-input single-output (SISO) system (this approach can be extended to multi-input multi-output (MIMO) systems as described in Section~\ref{sec:results}) identical to~\cite{Hovakimyan2010} for output feedback:
\begin{align}
    y_1(s) &= A(s) (u(s) + d_\lone(s))\, , \qquad y_2(s) = \tfrac{1}{s} y_1(s)\, , \label{eq:generalsystem}
\end{align}
where $y_1(s)$ and $y_2(s)$ are the Laplace transforms of the translational velocity $y_1(t)$, and position $y_2(t)$, respectively, $A(s)$ is a strictly-proper \emph{unknown} transfer function that can be stabilized by a proportional-integral controller, $u(s)$ is the Laplace transform of the input signal, and $d_\lone(s)$ is the Laplace transform of the disturbance signal defined as $d_\lone(t) \triangleq f(t,y_1(t))$, where $f : \Real \times \Real \rightarrow \Real $ is an \emph{unknown} map subject to the following assumption:

\begin{assumption}[Global Lipschitz continuity]
\label{as:lipshitz}
There exist constants $L > 0$ and $L_0 > 0$, such that the following inequalities hold uniformly in $t$:
\begin{align}
    | f(t,v) - f(t,w) | &\leq L | v - w | \, \text{, and} \label{eq:assumption1} \\
    | f(t,w) | &\leq L | w | + L_0 ~ \ \ \forall v,\,w \in \mathbb{R}\,. \label{eq:assumption2}
\end{align}
\end{assumption}
 

The system is tasked to track a desired postition trajectory $y_2^*(t)$, which is defined over a finite-time interval and is assumed to be feasible with respect to the true dynamics of the $\lone$-controlled system (Fig.~\ref{fig:cascadedL1architecture}, blue dashed box). This signal is discretized. We introduce the lifted representation, see \cite{Gunnarsson2001}, for the desired trajectory ${\bf{y^*_2}}=(y^*_2(1),\hdots,y_2^*(N))$, and the output of the plant ${\bf{y_2}}=(y_2(1),\hdots,y_2(N))$, where $N<\infty$ is the number of discrete samples. The tracking performance criterion $J$ is defined as: 
\[ J \triangleq \min_{\bf{e}} {\bf{e}}^T{\bf{Qe}}\]
where ${\bf{e}}={\bf{y}}_2-{\bf{y_2^*}}$ is the tracking error and $\bf{Q}$ is a positive definite matrix. The goal is to improve the tracking performance iteratively; that is, from execution to execution. 

\section{METHODOLOGY}
\label{sec:methodology}
We consider two main subsystems: the extended $\lone$ adaptive controller (blue dashed box in Fig.~\ref{fig:cascadedL1architecture}) and the ILC (red dashed box in Fig.~\ref{fig:cascadedL1architecture}). The extended $\lone$ adaptive controller is presented in Section~\ref{ssec:lone} including proofs of its transient behavior. Section~\ref{ssec:ilc} introduces the ILC. 

\subsection{$\lone$ Adaptive Control}
\label{ssec:lone}

In the proposed framework, the aim of the $\lone$ adaptive controller is to make the system behave in a repeatable, predefined way, even when unknown, changing disturbances affect the system. In this subsection, we describe the extended $\lone$ adaptive controller and provide proofs of the transient behavior.  


In this work, the typical $\lone$ adaptive output feedback controller for SISO systems~\cite{Hovakimyan2010} is nested within a proportional controller (see Fig.~\ref{fig:cascadedL1architecture}). This extended architecture is identical to~\cite{How2009}. The outer-loop proportional controller enables the system to remain within certain position boundaries. Given the proposed extended $\lone$ adaptive control, we must show that the system performs provably close to a given reference model under the uncertainty defined in Section~\ref{sec:problem}. This is done by finding bounds for the transient behavior. The proof is inspired by \cite{Hovakimyan2010}, but is extended to include the proportional controller (`Gain K' in Fig.~\ref{fig:cascadedL1architecture}). 

\subsubsection{Problem Formulation}


The objective of the extended $\lone$ adaptive output feedback controller is to design a control input $u(t)$  such that $y_2(t)$ tracks a bounded piecewise continuous reference input $r_2(t)$. To achieve this, one method is for the output of the $\lone$ adaptive controller nested within the proportional feedback loop $y_1(t)$ to track $r_1(t)$ according to a first-order reference system:
\begin{align}
    M(s) &= \tfrac{m}{s+m}\,, \quad m>0 \,. \label{eq:reference}
\end{align}


\subsubsection{Definitions and $\lone$-Norm Condition}
The system in~\eqref{eq:generalsystem} can be rewritten in terms of the reference system~\eqref{eq:reference}: 
\begin{align}
    y_1(s) &= M(s) (u(s) + \sigma(s)) \,, \label{eq:newsystem}
    \intertext{where uncertainties in $A(s)$ and $d_\lone(s)$ are combined into $\sigma$:}
    \sigma(s) &\triangleq \dfrac{(A(s) - M(s) )u(s) + A(s) d_\lone(s) }{M(s)} \,. \label{eq:sigma}
\end{align}
We consider a strictly-proper low-pass filter $C(s)$ (see Fig.~\ref{fig:cascadedL1architecture}) with $C(0) = 1$, and a proportional gain $K \in \Real^+$, such that:
\begin{align}
    H(s) &\triangleq \dfrac{ A(s) M(s) }{ C(s) A(s) + (1-C(s)) M(s) } \qquad \text{is stable,}\\
    F(s) &\triangleq \dfrac{1}{ s + H(s) C(s) K } \qquad \text{is stable,} \label{eq:defF}
    \intertext{and the following $\lone$-norm condition is satisfied:}
    \| G(s) &\|_{\lone} L < 1\, , \text{where} \label{eq:l1norm} ~
    G(s) \triangleq H(s) (1 - C(s)) F(s)  
\end{align}
and $L$ is the Lipschitz constant defined in Assumption~\ref{as:lipshitz}. 

The $\lone$-norm condition is used to prove bounded-input bounded-output (BIBO) stability of a reference model that will describe the repeatable behavior of the underlying $\lone$ controlled system. The solution of the $\lone$-norm condition in~\eqref{eq:l1norm} exists under the following assumptions:

\begin{assumption}[Stability of $H(s)$]
$H(s)$ is assumed to be stable for appropriately chosen low-pass filter $C(s)$ and first-order reference eigenvalue $-m<0$.
\end{assumption} 
As indicated in~\cite{Hovakimyan2010}, this assumption holds in cases where $A(s)$ can be stabilized by a proportional-integral controller.

\begin{assumption}[Stability of $F(s)$] 
$F(s)$ is assumed to be stable for appropriately chosen proportional gain $K$.
\end{assumption}

A sufficient condition for this assumption to be valid is if $A(s)$ is minimum phase stable, which holds if there is a controller within the system $A(s)$ that is stabilizing a plant without any unstable zeros. In the case of velocity control of a quadrotor, this assumption is valid. Less conservative conditions that guarantee the stability of $F(s)$ exist, but are not necessary for the application in this paper. 

\subsubsection{Extended $\lone$ Adaptive Control Architecture}

The SISO extended $\lone$ adaptive controller architecture is shown in Fig.~\ref{fig:cascadedL1architecture}. With the exception of the proportional feedback loop, this architecture (from $r_1$ to $y_1$) is identical to~\cite{Hovakimyan2010}. The integrator from $y_1$ to $y_2$ allows the outer-loop to control the position, while the $\lone$ adaptive feedback controls the velocity. The equations describing the implementation of the extended $\lone$ output feedback architecture are presented below in~\eqref{eq:outputpredictor}, \eqref{eq:adaptationlaw}, \eqref{eq:controllaw}, and~\eqref{eq:negfeedback}.

\begin{description}
\item [Output Predictor:] The following output predictor is used within the $\lone$ adaptive output feedback architecture:
   \[ \dot{\hat{y}}_1(t) = -m \hat{y}_1(t) + m (u(t) + \hat{\sigma}(t)) \,, \qquad \hat{y}_1(0) = 0 \,, \]
   where $\hat{\sigma}(t)$ is the adaptive estimate of $\sigma(t)$. In the Laplace domain, this is equivalent to:
    \begin{equation}
        \hat{y}_1(s) = M(s) (u(s) + \hat{\sigma}(s)) \,. \label{eq:outputpredictor}
    \end{equation}

\item [Adaptation Law:] The adaptive estimate $\hat{\sigma}(t)$ is updated according to the following update law:
\begin{align}
    \dot{\hat{\sigma}}(t) &= \Gamma \text{Proj} (\hat{\sigma}(t), -mP\tilde{y}(t)) \,, \qquad \hat{\sigma}(0) = 0 \,, \label{eq:adaptationlaw}
\end{align}
where $\tilde{y}(t) \triangleq \hat{y}_1(t) - y_1(t)$, and $P>0$ solves the algebraic Lyapunov equation $mP + Pm = 2mP = -Z$ for $Z>0~$. The variable $\Gamma \in \Real^+$ is the adaptation rate subject to the lower bound as specified in~\cite{Hovakimyan2010}. Typically in $\lone$ adaptive control, $\Gamma$ is set very large. Experiments with this controller were carried out with an adaptation rate of $\Gamma = 1000$. The projection operator defined in~\cite{Hovakimyan2010} ensures that the estimation of $\sigma$ is guaranteed to remain within a specified convex set.
\item [Control Law:]  The control input signal is the difference between the $\lone$ desired trajectory signal $r_1$ and the adaptive estimate $\hat{\sigma}$ after passing through the low-pass filter $C(s)$:
\begin{align}
    u(s) &= C(s) (r_1(s) - \hat{\sigma}(s)) \,. \label{eq:controllaw}
\end{align}

This means that only the low frequencies of the uncertainties within $A(s)$ and $d_\lone(s)$, which the system is capable of counteracting, are compensated for. The high frequency portion is attenuated by the low-pass filter.
\item [Closed-Loop Feedback:] The following equation describes the closed-loop feedback acting on the input to the $\lone$ adaptive output feedback controller $r_1$ based on the output of the system $y_1$. As discussed above: $y_2(s) \triangleq \frac{1}{s} y_1(s)$, and the negative feedback is defined as follows:
\begin{align}
    r_1(s) &= K (r_2(s) - y_2(s)) \,, 
    \label{eq:negfeedback}
\end{align}
where the objective is for $y_2$ to track $r_2$.
\end{description}

\subsubsection{Transient and Steady-State Performance}

The extended $\lone$ adaptive controller is required to perform repeatably and consistently. This is done by guaranteeing that the difference between the output of a known BIBO stable reference system and the output of the actual system is uniformly bounded. Intuitively, the reference system describes the desired behavior of the actual system.

The proof starts off by presenting a BIBO stable closed-loop reference system. This reference system is then compared to the actual extended $\lone$ adaptive output feedback controller. 

\begin{lemma}
\label{lm:referencesystem}
Let $C(s)$, $M(s)$ and $K$ satisfy the $\lone$-norm condition in~\eqref{eq:l1norm}. Then the following closed-loop reference system:
\begin{align}
    y_{2,\text{ref}}(s) &= F(s) H(s) \big(C(s) K r_2(s) + (1 - C(s)) d_{\text{ref}} (s) \big) \notag \\
    d_{\text{ref}} (t) &\triangleq f(t, y_{2,\text{ref}}(t)) \label{eq:clrefsystem}
\end{align}
is BIBO stable.
\end{lemma}

\begin{proof}
Since $r_2(t)$ is bounded and $H(s)$, $C(s)$ and $F(s)$ are strictly-proper stable transfer functions, taking the norm of the reference system and making use of Assumption~\ref{as:lipshitz} yields the following bound:
\begin{align}
    \| y_{2,\text{ref}_\tau} \|_{\linf} &\leq K \| H(s) C(s) F(s) \|_{\lone} \| r_2\|_{\linf} \notag \\
    & \qquad + \| G(s) \|_{\lone} (L \| y_{2,\text{ref}_\tau} \|_{\linf} + L_0) \,,
\end{align}
where $\| y_{2,\text{ref}_\tau} \|_{\linf}$ is the truncated $\linf$-norm of the signal $y_{2,\text{ref}}(t)$ up to $t = \tau$. 
    Let $\rho_r$ be defined as follows:
\begin{align}
    \rho_r &\triangleq \dfrac{ K \| H(s) C(s) F(s) \|_{\lone} \| r_2 \|_{\linf} + \| G(s) \|_{\lone} L_0 }{ 1 - \| G(s) \|_{\lone} L }\,. \label{eq:defnrhor}
\end{align} From the $\lone$-norm condition in~\eqref{eq:l1norm} and the definition of $\rho_r$ in~\eqref{eq:defnrhor}:
\begin{align}
    \| y_{2,\text{ref}_\tau} \|_{\linf} &\leq \rho_r\,.
\end{align}
This result holds uniformly, so $\| y_{2,\text{ref}} \|_{\linf}$ is bounded. Hence, the closed-loop reference system in~\eqref{eq:clrefsystem} is BIBO stable.
\end{proof}

\begin{theorem}
\label{th:bounds}
Consider the system in~\eqref{eq:generalsystem}, with a control input from the extended $\lone$ output feedback adaptive controller defined in~\eqref{eq:outputpredictor}, \eqref{eq:adaptationlaw}, \eqref{eq:controllaw}, and~\eqref{eq:negfeedback}. Suppose $C(s)$, $M(s)$ and $K$ satisfy the $\lone$-norm condition in~\eqref{eq:l1norm}. Then the following bounds hold:
\begin{align}
    \| \tilde{y} \|_{\linf} &\leq \gamma_0 \,, \label{eq:ytildebound} \\
    \| y_{2,\text{ref}} - y_2 \|_{\linf} &\leq \gamma_1 \,, \label{eq:overallbound}
\end{align}
where $\tilde{y}(t) \triangleq \hat{y}_1(t) - y_1(t)$, $\gamma_0 \propto \sqrt{\frac{1}{\Gamma}}$ is defined in~\cite{Hovakimyan2010}, and
\begin{align}
    \gamma_1 &\triangleq \dfrac{ \norm{ \dfrac{F(s) H(s) C(s) }{M(s)} }_\lone}{1 - \| G(s) \|_\lone L} \gamma_0 \,. \label{eq:defgamma1}
\end{align}
\end{theorem}

\begin{proof}
See Appendix.
\end{proof}

The bounds given in~\eqref{eq:ytildebound} and~\eqref{eq:overallbound} show that the difference between the output predictor and the system output $y_1(t)$ and the difference between the reference system and the system output $y_2(t)$ are uniformly bounded with bounds inversely proportional to the square root of the adaptation gain $\Gamma$. This means that for high adaptation gains, the actual system approaches the behavior of the reference system~\eqref{eq:clrefsystem}. Hence, the system achieves repeatable and consistent performance, which is required for ILC.

\subsection{Iterative Learning Control}
\label{ssec:ilc}
We use ILC to improve the tracking performance of the underlying, repeatable system. The algorithm updates the feedforward signal $r_2(t)$ based on data gathered during previous iterations. The ILC implementation in this work is based on~\cite{Schoellig2012}. In this subsection, we give a brief summary of the optimization-based ILC used in this work and highlight the differences to the approach in \cite{Schoellig2012}, where a more detailed description is found.

We consider a repeatable system as seen by the ILC, which includes both the plant and the extended $\lone$ adaptive controller (blue dashed box and shadowed box in Fig.~\ref{fig:cascadedL1architecture}), and whose key dynamics can be represented by the following model: 
\begin{equation}
\label{eq:nominalmodel}
 \begin{array}{c c}
  \dot{x}(t)=g(x(t),r_2(t))\,, & y_2(t) = h(x(t))\, ,
 \end{array}
\end{equation}
where $g$ and $h$ are nonlinear function, $r_2(t)\in \mathbb{R}$ is the control input to the system, $x(t)\in \mathbb{R}^{n_x}$ is the state and $y_2(t)\in\mathbb{R}$ is the output. To satisfy the typical ILC assumption of identical initial conditions, despite unknown disturbances, experiments start when the system state is in close vicinity of the desired initial state. This is possible as the $\lone$ adaptive controller compensates for the effect of unknown disturbances.

The desired output trajectory $y_2^*(t)$ is assumed to be feasible based on the nominal model~(\ref{eq:nominalmodel}), where ($r_2^*(t),\ x^*(t),\ y_2^*(t)$) satisfy~(\ref{eq:nominalmodel}). 
We assume that the system stays relatively close to the reference trajectory; hence, we only consider small deviations from the above nominal trajectories, $\tilde{r}_2(t)$, $\tilde{x}(t)$ and $\tilde{y}_2(t)$. The system is linearized about the nominal trajectories to obtain a time-varying, linear state-space model, which approximates the system dynamics along the reference trajectory. The system is discretized and rewritten in the lifted representation as in \cite{Schoellig2012}. We define $\bar{\bf{y}}_2=(\tilde{y}_2(1),\hdots,\tilde{y}_2(N))\in \mathbb{R}^{N}$ and analogously we define $\bar{\bf{r}}_2$.
The lifted representation for the extended system is written as: 
\begin{equation}
\label{eq:system}
 {\bar{\bf{y}}}_{2,j}={\bf{F}}_{\textnormal{ILC}}\bar{\bf{r}}_{2,j}+{\bf{d}}_j\,,
\end{equation}
where the subscript $j$ denotes the iteration number, ${\bf{F}}_{\textnormal{ILC}}$ is a constant matrix derived from the nominal model and $\bf{d}$ represents a repetitive disturbance that is initially unknown. 

Using the approach presented in~\cite{Mueller2012} and~\cite{Schoellig2012}, an iteration-domain Kalman filter for the system~(\ref{eq:system}) is used to compute the estimate $\widehat{\bf{d}}_{j|j}$ based on measurements from iterations 
$1,\hdots,j$. 

An optimization-based update step computes the next reference sequence $\bar{\bf{r}}_{2,j+1}$ that compensates for the identified disturbance $\widehat{\bf{d}}_{j|j}$ and estimated output error  $\widehat{\bf{y}}_{j+1|j}$, where $\widehat{\bf{y}}_{j+1|j} ={\bf{F}}_{\textnormal{ILC}} \bar{\bf{r}}_{2,j}+ \widehat{\bf{d}}_{j|j}$. In the input update step, the following quadratic cost function is minimized:
\begin{equation}
\label{eq:costfunction}
 \min_{\bar{\bf{r}}_{2,j+1}}\left( \widehat{\bf{y}}_{j+1|j}^T{\bf{Q}}\widehat{\bf{y}}_{j+1|j} + \bar{\bf{r}}_{2,j+1}^T{\bf{S}}\bar{\bf{r}}_{2,j+1} + {\bf{\ddot{\bar{r}}}}_{2,j+1}^T{\bf{R}}{\bf{\ddot{\bar{r}}}}_{2,j+1}\right)
\end{equation}
subject to 
\[   \ddot{\mathbf{\bar{r}}}_{2,j+1} \leq\mathbf{a}_{\max}\,,\] 
where $\mathbf{a}_{\max}$ is a constraint based on the maximum acceleration achievable by the physical system. The sequence ${\bf{\ddot{\bar{r}}}}_{2,j+1}$ represents the discrete approximation of the second derivative of the input reference. The constant matrices ${\bf{Q}}$, ${\bf{R}}$, ${\bf{S}}$ are symmetric positive definite matrices that weight different components of the cost function. 
The cost function tries to minimize the tracking error of the system (weighted by $\bf{Q}$), the control effort required (weighted by $\bf{S}$) and the rate of change of the reference signal derivative (weighted by $\bf{R}$). We use the IBM CPLEX optimizer to solve the above optimization problem. The cost function used in this work is different from the cost function in \cite{Schoellig2012} as it includes both the input and its second derivative to improve the performance of the given task.

In previous work (see \cite{Degen2014} and \cite{Lee2000}) the convergence for optimization-based ILC with Kalman filter, such as the one used in this paper, was proven. However, the cost function in \cite{Degen2014} and \cite{Lee2000} differs from the cost function in this paper. Instead  of including $\bar{\bf{r}}_{2,j+1}$ as in~(\ref{eq:costfunction}), the  cost function in~\cite{Degen2014} and \cite{Lee2000} only includes the reference input change from iteration to iteration $\Delta \bar{\bf{r}}_{2,j+1}= \bar{\bf{r}}_{2,j+1} - \bar{\bf{r}}_{2,j}$. Future work will extend the proof of \cite{Degen2014}, \cite{Lee2000} to our setup~(\ref{eq:costfunction}).

\section{EXPERIMENTAL RESULTS}
\label{sec:results}
The proposed framework combining $\lone$ adaptive control and ILC ($\lone$-ILC) is used to minimize the trajectory tracking error of a quadrotor flying a three-dimensional trajectory under different dynamic disturbances. The SISO architecture derived in the previous section is extended to the MIMO quadrotor system by implementing ($3\times 3$) diagonal transfer function matrices for the low-pass filter and first-order output predictor. The signals $r_1(t)$, $r_2(t)$, $y_1(t)$, and $y_2(t)$ are the desired translational velocity, desired position, quadrotor translational velocity and quadrotor position, respectively. This implementation is identical to~\cite{How2009}, which ensures that the quadrotor remains within the boundaries of the indoor flying space. Each element of the three-dimensional signals and each diagonal element of the transfer function matrices correspond to the $x$, $y$ and $z$ inertial directions, respectively.

The experiments were performed using the commercial quadrotor platform AR.Drone 2.0 from Parrot. An overhead motion capture camera system is used to obtain position information. To test the performance of the proposed approach under unknown, changing disturbances, we change the dynamic behavior of the quadrotor by adding a \emph{mass disturbance}. To create the mass disturbance a 50~g mass is suspended 55~cm below the back-left leg, 17 cm from the geometric center of the frame, creating a pendulum. 

We compare the performance of the proposed $\lone$-ILC approach with that of a pure ILC with an underlying, non-adaptive proportional-derivative controller (PD-ILC). 
To quantify the controller performance, the error in the system is defined as: 
\begin{equation}
 e=\dfrac{\sum_{i=1}^N \sqrt{(e_x(i))^2+(e_y(i))^2+(e_z(i))^2}}{N}
\end{equation}
where $e_x(i)=r_{2,x}^*(i)-y_{2,x}(i)$, $e_y(i)=r_{2,y}^*(i)-y_{2,y}(i)$ and $e_x(i)=r_{2,z}^*(i)-y_{2,z}(i)$ are the deviations from the desired trajectory in each axis. We consider three scenarios to compare the performance of the control frameworks:learning convergence and generalizability, repeatability, and performance under changing conditions. In all three scenarios the $\lone$-ILC approach outperforms the PD-ILC approach.

\begin{figure}[htp]
 \centering
  \begin{subfigure}[b]{0.5\textwidth}
   \centering
   \includegraphics[width=\textwidth]{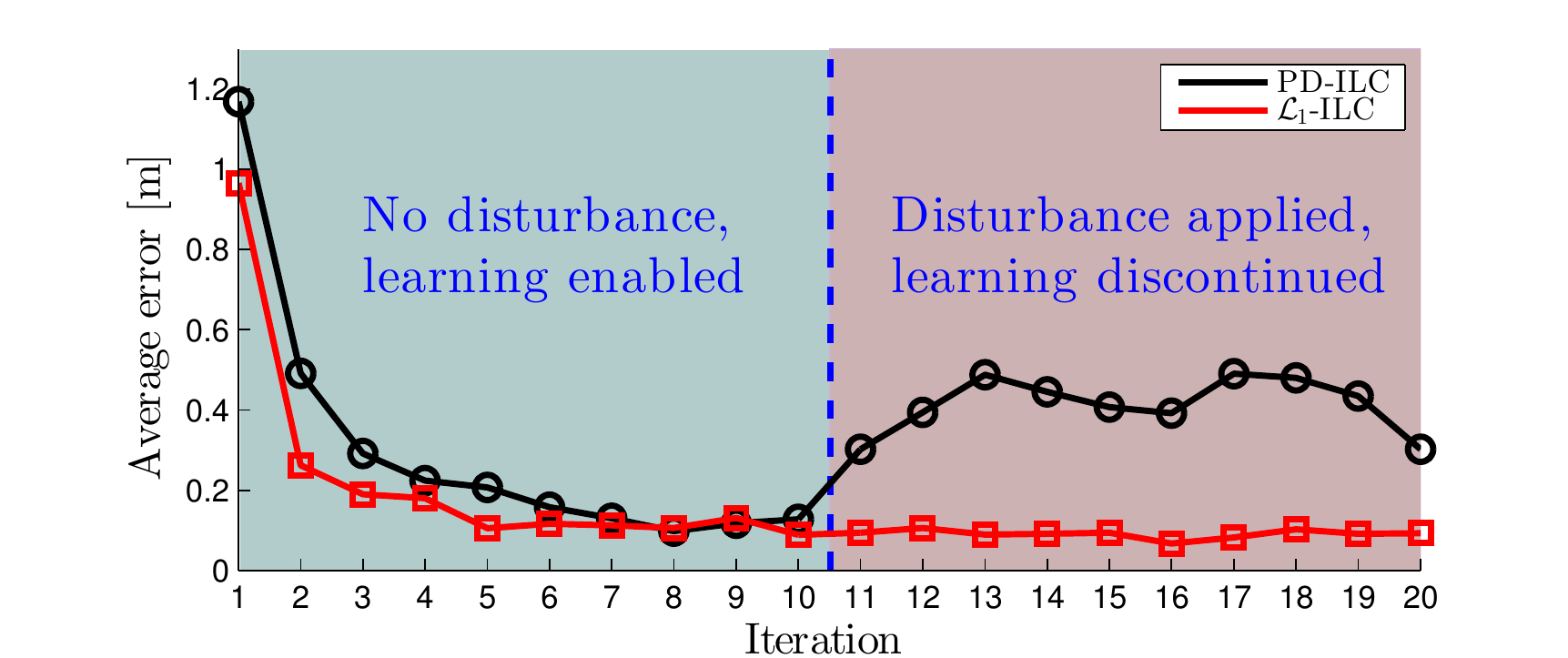}%
   \caption{}
   \label{fig:realerrorl1ncrepeat}%
  \end{subfigure}
  \begin{subfigure}[b]{0.5\textwidth}
   \centering
   \includegraphics[width=\textwidth]{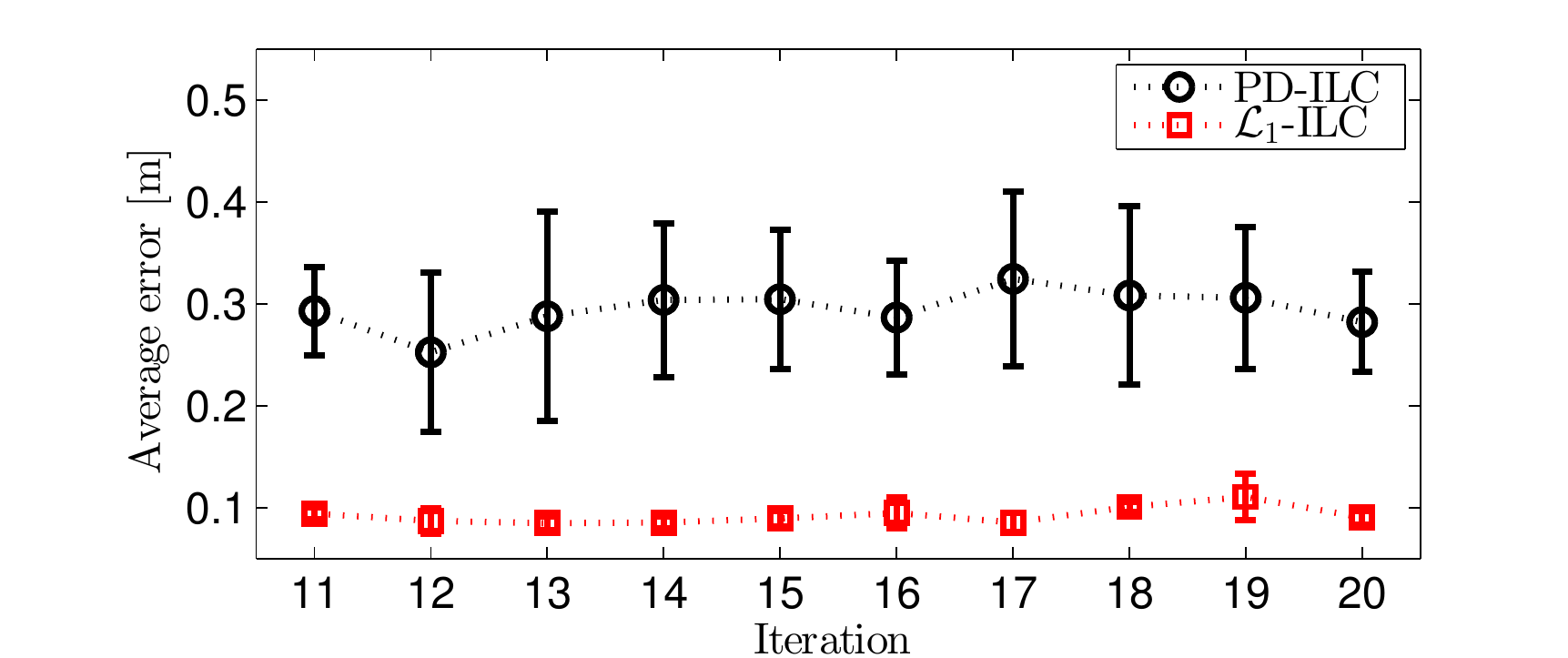}%
   \caption{}
   \label{fig:realerror5l1ncrepeat}%
  \end{subfigure}
 \caption{(\subref{fig:realerrorl1ncrepeat}) The $\lone$-ILC approach shows a faster learning convergence initially. At iteration 11 a disturbance is applied and learning is disabled: the $\lone$-ILC error is not affected while the PD-ILC error increases significantly.
 (\subref{fig:realerror5l1ncrepeat}) The mean of the error across five 10-iteration sets shows the repeatability of the learned trajectory after a mass disturbance is applied to the system. The PD-ILC approach displays a significantly larger error and standard deviation compared to the $\lone$-ILC approach.}
\end{figure}

\subsection{Learning Convergence and Generalizability}
The quadrotor learns to track a desired trajectory using each of the two frameworks: PD-ILC and $\lone$-ILC. The errors of this initial learning process (iteration 1-10) are depicted in Fig.~\ref{fig:realerrorl1ncrepeat}. The proposed $\lone$-ILC shows lower errors consistently and converges faster.

After this initial learning process a mass disturbance is applied to the system and the learning is discontinued. The learned trajectory at iteration ten is repeated for ten more iterations with both the $\lone$-ILC and PD-ILC framework, see  Fig.~\ref{fig:realerrorl1ncrepeat}. The PD-ILC framework shows a 323$\%$ increase after the mass disturbance is applied. The $\lone$-ILC approach shows no noticeable increase in the error because the $\lone$ adaptive controller achieves repeatable behavior, despite the disturbances applied to the system. 

\subsection{Repeatability}
\begin{figure}[htp]
 \centering
  \begin{subfigure}[b]{0.5\textwidth}
   \centering
   \includegraphics[width=\textwidth]{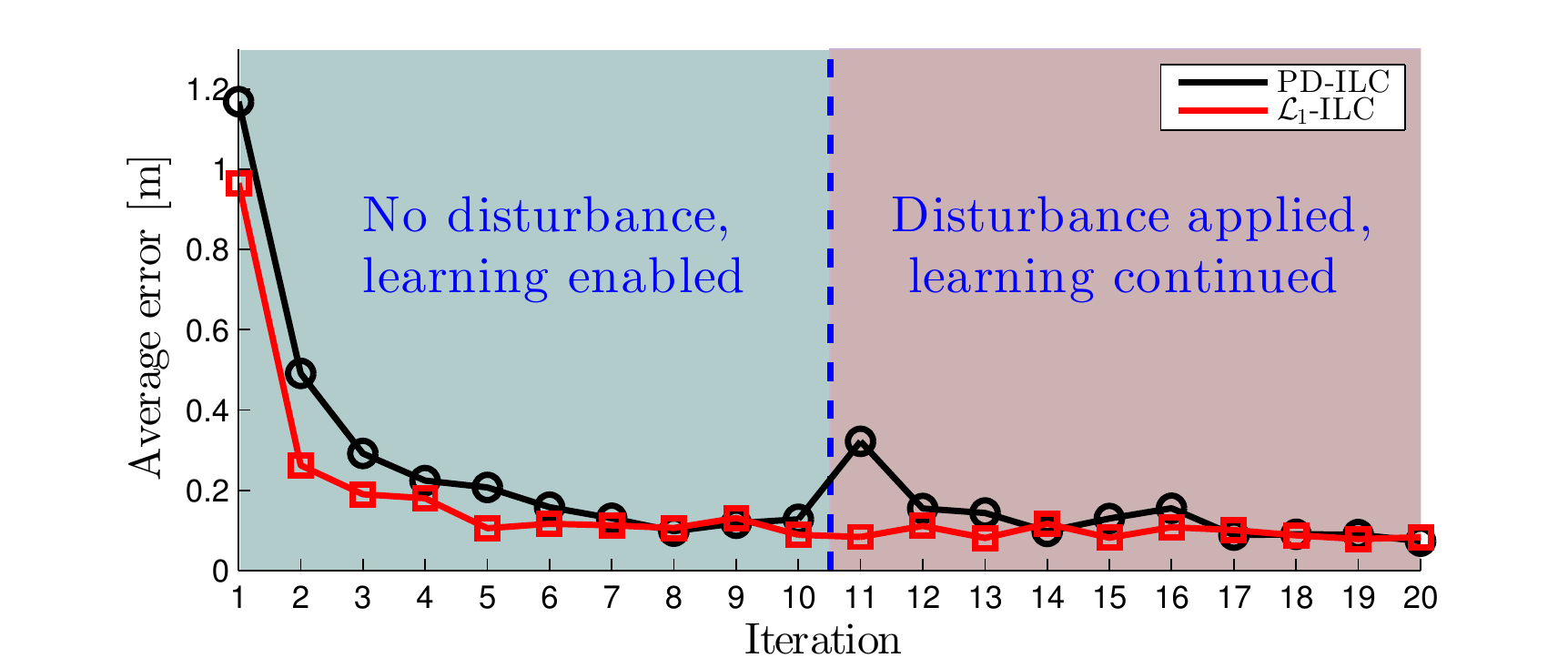}%
   \caption{}
   \label{fig:realerrorl1nclearn}%
  \end{subfigure}
  \begin{subfigure}[b]{0.5\textwidth}
   \centering
   \includegraphics[width=\textwidth]{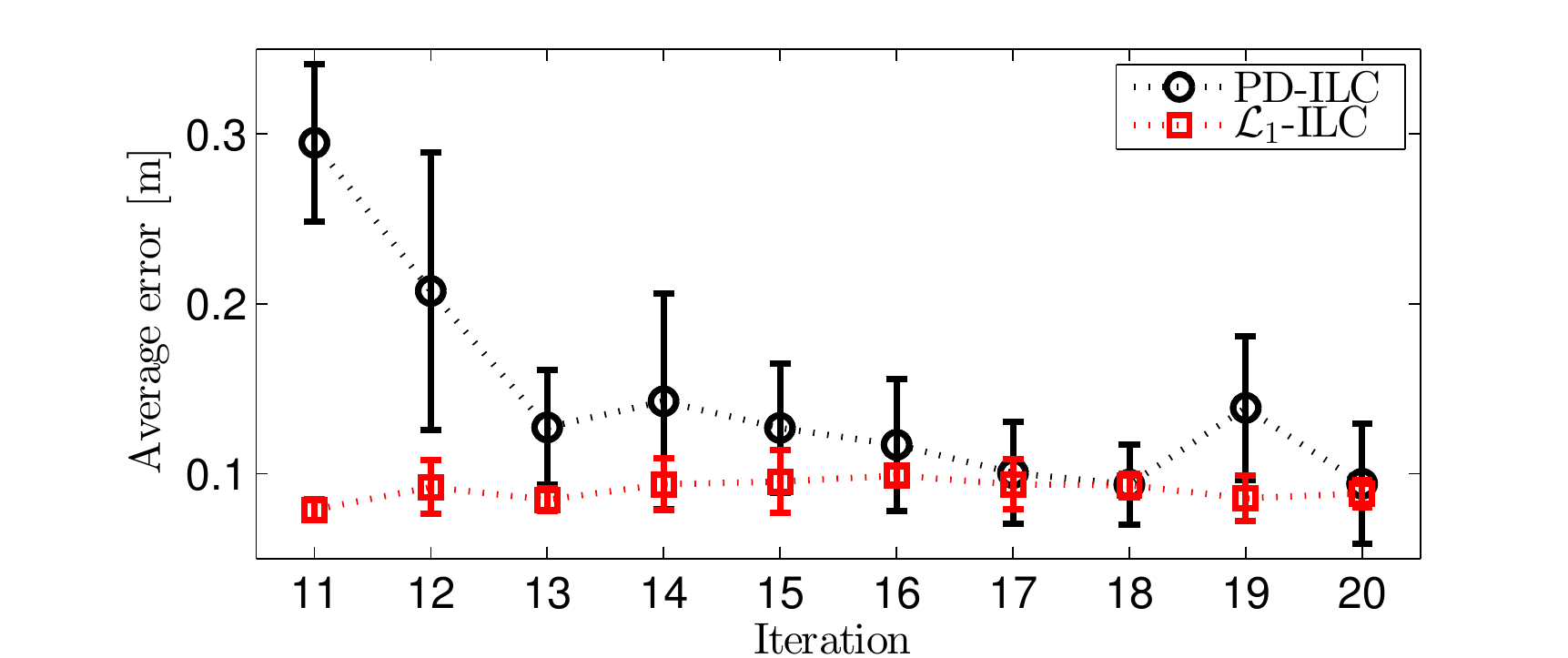}%
   \caption{}
   \label{fig:realerror5l1nclearn}%
  \end{subfigure}
 \caption{(\subref{fig:realerrorl1nclearn}) Learning behavior after a mass disturbance is applied to the system at the end of iteration ten. The error of the PD-ILC framework after the disturbance increases dramatically; while the error of the $\lone$-ILC framework is virtually unchanged. 
 (\subref{fig:realerror5l1nclearn}) Average error across five sets of ten iterations of learning after a mass disturbance is applied. The PD-ILC approach displays  a  significantly larger error and standard deviation than that of the $\lone$-ILC approach.}
\end{figure}

To assess the repeatability of the overall control after a mass disturbance has been applied to the system, we discontinued learning and performed five experiments with ten iterations each for both control frameworks. Fig.~\ref{fig:realerror5l1ncrepeat} shows the average error of the five sets at each iteration along with their standard deviation. The system is more repeatable with the $\lone$-ILC framework as the error and standard deviation are much smaller than with the PD-ILC framework.

\subsection{Performance under Changing Conditions}
The ability of the system to continue to learn after a disturbance has been applied is also explored. The errors while the system is learning without disturbance (first ten iterations) and with a mass disturbance (last ten iterations) are shown in  Fig.~\ref{fig:realerrorl1nclearn}. The error increases significantly in the PD-ILC framework after the  disturbance is applied. This error rapidly decreases as the system continues to learn; however, for some applications, this behavior may not be acceptable. The error in the $\lone$-ILC framework does not change even after the mass disturbance has been applied. 

The learning behavior is further explored by obtaining a total of five 10-iteration sets of the learning systems after the mass disturbance is applied. The average of the error and the standard deviation across the five sets are shown in Fig.~\ref{fig:realerror5l1nclearn}. The average error at iteration eleven for the PD-ILC framework is significantly higher than for the $\lone$-ILC framework. The standard deviation is notably higher for the PD-ILC approach for all iterations. The $\lone$-ILC experiments show that the learned input trajectory can be re-used even if the system dynamics are changed.  

\section{CONCLUSIONS}
\label{sec:conclusions}
In this paper, we introduced an $\lone$-ILC framework for trajectory tracking. The $\lone$ adaptive controller forces the system to remain close to a predefined nominal system behavior, even in the presence of unknown and changing disturbances. However, having a repeatable system does not imply achieving zero tracking error. We use ILC to learn from previous iterations and improve the tracking performance over time. We proved that the proposed framework is stable and achieves learning convergence. Experiments on quadrotors showed significant performance improvements of the proposed $\lone$-ILC approach compared to a non-adaptive PD-ILC approach in terms of learning convergence, repeatability, and behavior under disturbances. The learned reference trajectories of the $\lone$-ILC framework are re-usable even if the system dynamics are changed, because the $\lone$ adaptive controller compensates for the unknown, changing disturbances. As far as the authors are aware, this is the first work to show such an $\lone$-ILC framework in real-world experiments and on quadrotor vehicles, specifically. 
\addtolength{\textheight}{-7cm}  
\appendix{

Below we sketch the proof of Theorem~\ref{th:bounds}:
\begin{proof}
Theorem 4.1.1 in~\cite{Hovakimyan2010} proves the bound in~\eqref{eq:ytildebound} under the same assumptions as made in this paper. The bound in~\eqref{eq:overallbound} remains to be shown. The following definitions will become useful:
\begin{align}
    H_0 (s) &\triangleq \dfrac{A(s)}{C(s)A(s) + (1 - C(s)) M(s)} \label{eq:defH0} \,\text{, and} \\
    H_1 (s) &\triangleq \dfrac{ ( A(s) - M(s) ) C(s) }{C(s)A(s) + (1 - C(s)) M(s)} \,. \label{eq:defH1} 
    \intertext{In~\cite{Hovakimyan2010}, it is shown that both $H_0(s)$ and $H_1(s)$ are strictly-proper stable transfer functions. Furthermore, the following expressions using~\eqref{eq:defH0} and~\eqref{eq:defH1} can be verified:}
    M(s) H_0(s) &= H(s) \,\text{, and} \label{eq:proofexpression1}\\
    M(s) \Big( C(s) + &H_1(s) (1 - C(s)) \Big) = H(s) C(s) \,. \label{eq:proofexpression2}
\end{align}

Let $\tilde{\sigma}(t) \triangleq \hat{\sigma} (t) - \sigma(t)$ where $\hat{\sigma}$ is the adaptive estimate, and $\sigma$ is defined in~\eqref{eq:sigma}. The control law in~\eqref{eq:controllaw} can be expressed as:
\begin{align}
    u(s) &= C(s) r_1(s) - C(s) (\tilde{\sigma}(s) + \sigma(s)) \,. \label{eq:proofcontrol} \\
    \intertext{Substitution of~\eqref{eq:proofcontrol} into~\eqref{eq:sigma} and making use of the definitions in~\eqref{eq:defH0} and~\eqref{eq:defH1} results in the following expression for $\sigma(s)$:}
    \sigma(s) &= H_1(s) (r_1(s) - \tilde{\sigma}(s)) + H_0(s) d_\lone(s) \,. \label{eq:proofsigma}
    \intertext{Substitution of~\eqref{eq:proofcontrol} and~\eqref{eq:proofsigma} into the system~\eqref{eq:newsystem} results in:}
    y_1(s) &= M(s) \Big( C(s) + H_1(s) (1 - C(s)) \Big) \big( r_1(s) - \tilde{\sigma}(s) \big) \notag \\
    & \qquad + M(s) H_0(s) (1 - C(s)) d_\lone(s) \,.  \notag
    \intertext{From~\eqref{eq:proofexpression2} and~\eqref{eq:proofexpression1}, this expression simplifies to:}
    y_1(s) &= H(s) C(s) \big( r_1(s) - \tilde{\sigma}(s) \big) + H(s) (1 - C(s)) d_\lone(s) \,. \label{eq:finproofsystem}
\end{align}
An expression for $y_2$ is obtained by substituting~\eqref{eq:finproofsystem} and~\eqref{eq:negfeedback} into $y_2(s) = \frac{1}{s} y_1(s)$ and making use of the definition in~\eqref{eq:defF}:
\begin{align}
    y_2(s) &= F(s) H(s) \Big( C(s) K r_2(s) + (1 - C(s)) d_\lone(s) \Big) \notag \\
    & \qquad - F(s) H(s) C(s) \tilde{\sigma}(s)\,. \label{eq:cllonesystem}
\end{align}
Substitution of~\eqref{eq:outputpredictor} and~\eqref{eq:newsystem} into the definition of $\tilde{y}$ in the adaptation law results in the following expression for $\tilde{y}(s)$:
\begin{align}
    \tilde{y}(s) &= M(s) \tilde{\sigma}(s)\,. \label{eq:prooftildey}
\end{align}
Recalling the reference system in~\eqref{eq:clrefsystem} and using the expression for $y_2$ in~\eqref{eq:cllonesystem}, the error between reference and actual systems, $y_{2,\text{ref}} - y_2$ is:
\begin{align}
    y_{2,\text{ref}}(s) - y_2(s) &= F(s) H(s) \big( 1 - C(s) \big) ( d_{\text{ref}} (s) - d_\lone(s) ) \notag \\
    & \qquad - \dfrac{F(s) H(s) C(s)}{M(s)} M(s) \tilde{\sigma}(s) \,. \notag
    \intertext{Substituting the expression for $\tilde{y}(s)$ in~\eqref{eq:prooftildey} and the definition of G(s) in~\eqref{eq:l1norm}, we obtain:}
    y_{2,\text{ref}}(s) - y_2(s) &= G(s) ( d_{\text{ref}} (s) - d_\lone(s) ) \notag \\
    & \qquad - \dfrac{F(s) H(s) C(s)}{M(s)} \tilde{y}(s) \,. \notag
    \intertext{Finally, since the $\lone$-norm of G(s) exists, and $\frac{F(s) H(s) C(s)}{M(s)}$ is strictly proper and stable, the following bound can be derived by taking the truncated $\linf$-norm and by making use of Assumption~\ref{as:lipshitz}:}
    \big\| y_{2,\text{ref}_t} - y_{2_t} \big\|_\linf &\leq \big\| G(s) \big\|_\lone L \big\| y_{2,\text{ref}_t} - y_{2_t} \big\|_\linf \notag \\
    & \qquad + \bigg\| \dfrac{F(s) H(s) C(s)}{M(s)} \bigg\|_\lone \big\| \tilde{y}_t \big\|_\linf \notag\\
    &\leq \dfrac{\bigg\| \dfrac{F(s) H(s) C(s)}{M(s)} \bigg\|_\lone}{1 - \big\| G(s) \big\|_\lone L} \big\| \tilde{y}_t \big\|_\linf \,, \notag
    \intertext{which holds uniformly. From the bound in~\eqref{eq:ytildebound} proven in~\cite{Hovakimyan2010}, the following bound is derived:}
    \big\| y_{2,\text{ref}} - y_{2} \big\|_\linf &\leq \dfrac{\bigg\| \dfrac{F(s) H(s) C(s)}{M(s)} \bigg\|_\lone}{1 - \big\| G(s) \big\|_\lone L} \gamma_0 = \gamma_1 \,, \notag
\end{align}
proving the second bound in~\eqref{eq:overallbound}.
\end{proof}



\bibliographystyle{IEEEtran}
\bibliography{thisbib}

\end{document}